\documentclass{elsarticle}

\makeatletter
\def\ps@pprintTitle{%
 \let\@oddhead\@empty
 \let\@evenhead\@empty
 \def\@oddfoot{}%
 \let\@evenfoot\@oddfoot}
\makeatother

\usepackage[english]{babel}

\usepackage{xspace}
\usepackage{comment}
\usepackage[all]{xy}

\usepackage{amssymb}
\usepackage{amsmath}
\usepackage{amsthm}
\usepackage{stmaryrd}
\usepackage{microtype}

\theoremstyle{plain}
\newtheorem{theorem}{Theorem}[section]
\newtheorem{proposition}[theorem]{Proposition}

\theoremstyle{definition}
\newtheorem{definition}{Definition}[section]
\newtheorem{example}{Example}[section]
\newtheorem{remark}{Remark}[section]

\newcommand{\N}{\mathbb{N}}
\newcommand{\R}{\mathbb{R}}
\newcommand{\Parts}[1]{\mathcal{P}\left(#1\right)}
\newcommand{\Set}[1]{\lbrace #1 \rbrace}
\newcommand{\DownSets}[1]{\mathcal{O}\left(#1\right)}
\newcommand{\cl}[1]{\overline{#1}}
\newcommand{\clOp}{\overline{\mbox{}\cdot\mbox{}}}
\newcommand{\clx}[1]{\llbracket #1 \rrbracket}
\newcommand{\clxOp}{\llbracket\cdot\rrbracket}
\newcommand{\mcl}[1]{\mu_{n,\ell}\left(#1\right)}
\newcommand{\mclOp}{\mu_{n,\ell}}
\newcommand{\mclIter}[2]{\mu_{n,\ell}^{#2}\left(#1\right)}

\newcommand{\Ixo}[1]{\mathcal{R}_{I,#1}}
\newcommand{\Pxo}[1]{\mathcal{R}_{P,#1}}
\newcommand{\SC}[1]{\text{SC$_{\ell,#1}$}}

\newcommand{\Pop}{\mathbf{P}}

\newcommand{\rx}[1]{r_x\left(#1\right)}
\newcommand{\rxOp}{r_x}
\newcommand{\alt}[1]{\text{alt}\left({#1}\right)}
\newcommand{\ie}{i.e.,\xspace}
\newcommand{\eg}{e.g.,\xspace}
\newcommand{\wrt}{w.r.t.\xspace}
\newcommand{\Cech}{\v{C}ech\xspace}

\title{A Distance Between Populations\\ for $n$-Points Crossover in Genetic Algorithms}

\begin{document}

\author[isegi]{Mauro Castelli}
\ead{mcastelli@novaims.unl.pt}

\author[disco]{Gianpiero Cattaneo}
\ead{cattang@live.it}

\author[disco]{Luca Manzoni}
\ead{luca.manzoni@disco.unimib.it}

\author[isegi]{Leonardo Vanneschi}
\ead{lvanneschi@novaims.unl.pt}

\address[isegi]{%
  NOVA IMS, Universidade Nova de Lisboa, 1070-312 Lisboa, Portugal}

\address[disco]{%
  Dipartimento di Informatica, Sistemistica e Comunicazione,\\
  Universit\`a degli Studi di Milano-Bicocca, Milano, Italy}

\begin{abstract}
  Genetic algorithms (GAs) are an optimization technique that has been successfully used on many real-world problems. There exist different approaches to their theoretical study. In this paper we complete a recently presented approach to model one-point crossover using pretopologies (or \Cech topologies) in two ways. First, we extend it to the case of $n$-points crossover. Then, we experimentally study how the distance distribution changes when the number of crossover points increases.
\end{abstract}

\maketitle

\section{Introduction}
\label{sec:introduction}

The usual approach to the study of genetic algorithms (GAs) is to model their dynamics either using some simple kind of
crossover, like the one-point crossover, or without focusing on the difference given by using different kinds of
crossover (see, for instance~\cite{RR02} for a comprehensive overview). Several models have appeared so far that study
the dynamics under one-point crossover or in which the crossover does not play a significant role. The most prominent
example is the work of Vose and his coworkers~\cite{Vo98,Vo10}, that is based on modeling GA -- with selection,
mutation, and crossover -- as a deterministic system, under the hypothesis of an infinite population.

Topology has often been considered an important concept in the study of GAs. In the work of
Moraglio~\cite{MP04,Mo10,Mo11}, the crossover and mutation operators are defined and studied in terms of
topology, generally induced by a metric space, making the model applicable to a wide range of evolutionary computation
techniques. In this model, some general results on crossover are reached, for example on the points that can be
generated by a \emph{geometric} crossover (a condition respected by many crossovers used in practice). Among the
possible approaches, the definition of an operator-based distance for evolutionary computation techniques is an
important step in analyzing some aspects of the dynamics of the evolutionary algorithm~\cite{MO11b}. It is useful, for
example, for computing indicators of problem difficulty (\eg fitness distance correlation~\cite{JF95,Va04,TV05}). Pretopologies, instead of topologies, have been also used to study the dynamics of GAs. The first use of pretopologies to model crossover is due to Stadler, Wagner and coworkers~\cite{SW97,SS02}, where a connection with hypergraphs is also made.

The most studied type of crossover is the one-point crossover. Modeling only one crossover point is, of course,
unsatisfactory, since GAs are inspired by real biological processes in which more than one crossover point exists.
Furthermore, since GAs are commonly used in optimization, it is appropriate to study operators that generalize one-point
crossover. The study of these operators results in a better understanding of the impact of the used type of crossover on the performance of GAs. Of course, the generalization of one-point crossover to $n$-points crossover is not a trivial, both computationally and from the complexity of the mathematical formalization perspective.

As previously noted, an important mathematical tool, also used in the case of one point crossover, is the notion of metric space, obviously induced from an appropriate distance. This observation led to the study and definition of a
population-based crossover distance for one-point crossover~\cite{MV12}. This work is a natural generalization
of~\cite{MV12} to $n$-points crossover. While a distance ``consistent'' with traditional GAs mutation is the Hamming
one, defining a distance ``consistent'' with crossover is clearly more difficult. In fact, differently from mutation,
when an individual is fixed, the result of a crossover operation depends on the entire population.

When used as an optimization algorithm, GAs are inherently stochastic. However, there is a lower bound to the number of
generations needed to transform a given population into another one. This lower bound can be found in a deterministic
way, selecting an ``optimal'' sequence of populations. The length of this sequence allows us to determine what is the
minimum number of generations needed for passing from a population to another. This lower bound also holds for the commonly used stochastic version of GA. In order to find this lower bound, we use sets of populations, while, as said above, standard GAs have a single population that, differently from here, is modified by the genetic operators in a stochastic manner.
In other words, in this stochastic context, one obtains, after one step, different populations with different
probabilities. The two dynamics are related in the sense that the deterministic process that we study keeps track of all
reachable population (\ie the ones that have a positive probability of being obtained) with the aim of determining when
a population becomes reachable. Here we are not proposing a new version of GA for optimization, but a method to single
out the dynamics of existing GAs when all the stochastic choices are, in some sense, optimal. We have experimentally verified that a number of different crossover points can change, on average, this bound. When the number of crossover points increases there is a measurable reduction in the average distance between two populations. This result is a stimulus to continue the investigation of the effect that different types of crossovers can have on the dynamics of the population.

This work is structured as follows. In Section~\ref{sec:basic_notions} the basic model definitions and the mathematical
notions needed in the subsequent parts of the paper are recalled. In Section~\ref{sec:crossover_relation} we give a
generalization to the $n$-points crossover of the model introduced in~\cite{MV12}. The main results are presented in
Section~\ref{sec:representation_for_populations}. In Section~\ref{sec:experiments} we experimentally study how the distance distribution varies for different numbers of crossover points. The paper ends with the final remarks of Section~\ref{sec:further_remarks} and with some proposals for possible future works.

\section{Basic Notions}
\label{sec:basic_notions}
In this section, some basic notations and notions about lattices and closure operators are introduced.

We denote by $[i,j]$ (resp. $(i,j)$) with $i,j \in \N$ the set $\Set{i,i+1,\ldots,j-1,j} \subseteq \N$ (resp.
$\Set{i+1,\ldots,j-1}$). The meaning of $[i,j)$ and $(i,j]$ is the immediate extension of the previous notation. For a
fixed $\ell \in \N$, we denote by $\SC{1}$ the set $\Set{[i,j] \; | \; 1 \le i \le j \le \ell}$. In a previous
paper~\cite{MV12} it has been proved that $\SC{1}$ is a lattice \wrt the set inclusion. Also, for every $n \in \N$ with
$n>1$, we denote by $\SC{n}$ the set of all subsets of $[1,\ell]$ that can be written as the union of a set of
$\SC{n-1}$ and a set of $\SC{1}$.

A finite alphabet will be denoted by $\Sigma$. The set of all the strings of a given length $\ell$ composed of symbols
from $\Sigma$ is denoted by $\Sigma^\ell$. An element $x \in \Sigma^\ell$ is denoted by $x_1,\ldots,x_\ell$. The
notation $x_{[i,j]}$ is a short-cut for $x_i,x_{i+1},\ldots,x_{j-1},x_j$.

A function $f$ between two partially ordered sets $A$ and $B$ is said to be \emph{isotone} or \emph{order-preserving}
iff $\forall a,b \in A\; a \le b \Rightarrow f(a) \le f(b)$.

Recall that a \emph{lattice} $\mathcal{L}$ is a non-empty set $L$ endorsed with a partial ordering $\le_L$ such that for
any two elements $a,b \in L$ the \emph{join} $a \vee b$ (\ie the least upper bound of $a$ and $b$) and the \emph{meet}
$a \wedge b$ (\ie the greatest lower bound of $a$ and $b$) operators are uniquely defined in $L$~\cite{Bi67}. A lattice
is \emph{bounded} if $\bigvee L$ (\ie the maximal element for $L$) and $\bigwedge L$ (\ie the minimal element for $L$)
exist. A lattice is \emph{complete} if for every subset $S$ of $L$ then both $\bigvee S$ and $\bigwedge S$ exist. Note
that every finite lattice is complete.

Given a partially ordered set (poset) $\mathcal{L} = (L,\le_L)$ a subset $O$ of $L$ is a \emph{lower set} if for all $x
\in L$ condition ``there exists $y \in O$ with $x \le_L y$'' implies ``$x \in O$''. The set of all lower sets of a poset
$\mathcal{L}$ is denoted by $\DownSets{\mathcal{L}}$ and it is a complete lattice with respect to the set inclusion.
See~\cite{Bi67} for a reference on lattices.

A \emph{\Cech closure}~\cite{Ce66} on a set $X$ is a function on the power set of $X$, $\clOp:\Parts{X}\mapsto\Parts{X}$, such that:
\begin{enumerate}
  \item $\cl{\emptyset} = \emptyset$
  \item $\forall A \subseteq X \quad A \subseteq\cl{A}$
  \hfill \emph{(monotonicity)}
  \item $\forall A,B \subseteq X \quad \cl{A} \cup
  \cl{B} = \cl{A \cup B}$ \hfill
  \emph{(additivity)}
\end{enumerate}
Recall that a \emph{Kuratowski closure} is a \Cech closure with the following additional condition:
\begin{enumerate}
  \item[4.] $\forall A \subseteq X \; \cl{A} =
  \cl{\cl{A}}$ \hfill \emph{(idempotency)}
\end{enumerate}
The Kuratowski closure is one of the ways to define a topology on $X$~\cite{Mu99}. A \Cech closure can be iterated, defining
$\cl{A}^{i}$ with $i \in \N$ as follows:
\begin{equation}
\cl{A}^i =
\begin{cases}
  \cl{\cl{A}^{i-1}} & \text{if } i \ne 0\\
  A & \text{otherwise}
\end{cases}
\label{eq:iterated-closure}
\end{equation}
When $X$ is finite the function $\clxOp:\Parts{X}\to\Parts{X}$ defined as $\clx{A} = \bigcup_{i \in \N}\cl{A}^i$ is a
\emph{Kuratowski closure}.

\section{An Extension of  the Model  to $n$-Points Crossover}
\label{sec:crossover_relation}
In this section, we extend to $n$-points crossover the one-point model presented in~\cite{MV12}. In order to keep this work self-contained, we present the adapted definitions even when they remain similar to the ones already existing
for one-point crossover. The proofs of the propositions of this section can be obtained by a generalization of the proofs of~\cite{MV12}.

The model that we are going to define is based on the idea that, given two populations $P_1$ and $P_2$, with $P_2$ reachable from $P_1$, of a GA it is possible to count the minimum number of generations needed to transform $P_1$ in $P_2$ using only crossover operations. Hence, we decided to not consider, for now, the fact that GA has an essential stochastic component. The semantics of the two populations $P_1$ and $P_2$ is the following: the first population, $P_1$, is the initial
population of the GA and the second one, $P_2$, is the target - or final - population, containing one optimal solution.
Hence, the minimum number of generations needed to go from $P_1$ to $P_2$ represents a lower bound on the number of generations needed by a GA (even when it has a stochastic component) to reach an optimal solution.

\subsection{Crossover relation}

A first step in the introduction of our simplified model of GA with $n$-points crossover is the definition of a
\emph{crossover relation}. The simplified aspect of this model is that a population is any possible subset of strings of
a fixed length $\ell$ over an alphabet $\Sigma$, in which both the fixed population size and the presence of duplicate
elements are ignored.

\begin{definition}
  A $n$-points crossover relation (for $n \in [1,\ell)$) is a binary relation $\Ixo{n}$ over $\Sigma^\ell \times \Sigma^\ell$ such that:\\
  $\forall x,y,x',y' \in \Sigma^\ell$, $(x,y) \Ixo{n} (x',y')$ iff $\exists k_0=0,k_1,\ldots,k_n,k_{n+1}=\ell \in \N$ (not necessarily all distinct) such that $\forall i \in [0,n]$:
  \[
  x'_{(k_i,k_{i+1}]} =
  \begin{cases}
    x_{(k_i,k_{i+1}]} & \text{if } i \text{ is even}\\
    y_{(k_i,k_{i+1}]} & \text{otherwise}
  \end{cases}
  \]
  and
  \[
  y'_{(k_i,k_{i+1}]} =
  \begin{cases}
    y_{(k_i,k_{i+1}]} & \text{if } i \text{ is even}\\
    x_{(k_i,k_{i+1}]} & \text{otherwise}
  \end{cases}
  \]
\end{definition}

In the relation $\Ixo{n}$ the symbol $I$ refers to ``individuals''.

For the case of one-point crossover (\ie $n=1$) this definition is exactly the one presented in~\cite{MV12}.
Intuitively, we have that two pairs of elements of $\Sigma^\ell$ are in relation \wrt this definition if the second pair
can be obtained from the first one using one $n$-point crossover operation. The relation $\Ixo{n}$ is reflexive and
symmetric but not transitive (\ie following~\cite{Po800,Po903}, it is a \emph{similarity relation}).

\begin{example}
  \label{ex:xo_rel}
  Let consider $\Sigma = \{0,1\}$ and $\ell = 6$. An example of two pairs of strings in relation \wrt $4$-points crossover is the following:
  \[
  \left(
    \begin{array}{cccccc}
      0 & 1 & 0 & 0 & 0 & 1\\
      1 & 0 & 1 & 1 & 0 & 0
    \end{array}
  \right)
  \Ixo{4}
  \left(
    \begin{array}{c|c|c|cc|c}
      0 & 0 & 0 & 1 & 0 & 1\\
      1 & 1 & 1 & 0 & 0 & 0
    \end{array}
  \right)
  \]
  Notice that the relation $\Ixo{4}$ is not transitive. For example the pair
  \[
  \left(
    \begin{array}{cccccc}
      0 & 0 & 0 & 0 & 0 & 0\\
      1 & 1 & 1 & 1 & 1 & 1
    \end{array}
  \right)
  \Ixo{4}
  \left(
    \begin{array}{c|c|c|c|cc}
      0 & 1 & 0 & 1 & 0 & 0\\
      1 & 0 & 1 & 0 & 1 & 1
    \end{array}
  \right)
  \]
  and the pair
  \[
  \left(
    \begin{array}{cccccc}
      0 & 1 & 0 & 1 & 0 & 0\\
      1 & 0 & 1 & 0 & 1 & 1
    \end{array}
  \right)
  \Ixo{4}
  \left(
    \begin{array}{cccc||c|c|}
      0 & 1 & 0 & 1 & 0 & 1\\
      1 & 0 & 1 & 0 & 1 & 0
    \end{array}
  \right)
  \]
  are both in the relation $\Ixo{4}$ (notice that the first two crossover points - denoted by a double line - coincide, i.e., $k_1 = k_2$), but the pair
  \[
  \left(
    \begin{array}{cccccc}
      0 & 0 & 0 & 0 & 0 & 0\\
      1 & 1 & 1 & 1 & 1 & 1
    \end{array}
  \right)
  \text{ and }
  \left(
    \begin{array}{cccccc}
      0 & 1 & 0 & 1 & 0 & 1\\
      1 & 0 & 1 & 0 & 1 & 0
    \end{array}
  \right)
  \]
  is not in the relation $\Ixo{4}$.
\end{example}
Notice that the choice of starting with a swap on the first interval or in the second one (i.e., at odd or even crossover points), is actually immaterial. That is, in the first case of the previous example we would have:
\[
  \left(
    \begin{array}{cccccc}
      0 & 1 & 0 & 0 & 0 & 1\\
      1 & 0 & 1 & 1 & 0 & 0
    \end{array}
  \right)
  \Ixo{4}
  \left(
    \begin{array}{c|c|c|cc|c}
      1 & 1 & 1 & 0 & 0 & 0\\
      0 & 0 & 0 & 1 & 0 & 1
    \end{array}
  \right)
\]
That this, the two individuals obtained are still the same.

This relation has been extended to the power set of $\Sigma^\ell$ as follows.

\begin{definition}
  A $n$-point crossover relation $\Pxo{n}$ over $\Pop = \Parts{\Sigma^\ell}$ is a relation such that $\forall P_1,P_2 \in \Pop$:
  \[
  \begin{array}{rcl}
    P_1 \Pxo{n} P_2  & \Leftrightarrow & \forall x' \in P_2 \; \exists y' \in \Sigma^\ell \; \exists x,y \in P_1\\
    & & \text{s.t } (x,y) \Ixo{n} (x',y')
  \end{array}
  \]
\end{definition}

In the relation $\Pxo{n}$ the symbol $P$ refers to ``populations''.

For $n=1$ this definition is the same as the one given in~\cite{MV12}. Intuitively, two sets are in the relation $\Pxo{n}$ if every element of the second set can be obtained by using $n$-point crossover operations starting from elements of the first set. It is immediate that $\Pxo{n}$ is reflexive, but neither symmetric nor transitive. However, the following property holds:\\
\[
\forall P_1,P_2 \in \Pop, P_1 \Pxo{n} P_2 \text{ implies that } \forall P'_1 \supseteq P_1 \text{ and } \forall P'_2 \subseteq P_2, P'_1 \Pxo{n} P'_2\qquad.
\]
In order to clarify this property, let us discuss the
following example:
\begin{example}
  Let $P_1, P_2 \subseteq \{0,1\}^3$ be:
  \begin{align*}
    & P_1 = \lbrace (0,1,0), (1,0,1) \rbrace
    & P_2 = \lbrace (1,1,1), (0,0,0) \rbrace
  \end{align*}
  and $P_1',P_2' \subseteq \{0,1\}^n$ be:
  \begin{align*}
    & P_1' = \lbrace (0,1,0), (1,0,1), (1,1,0) \rbrace
    & P_2' = \lbrace (0,0,0) \rbrace
  \end{align*}
  When considering only two-points crossover one obtains $P_1 \Pxo{2} P_2$. That is, with only application of two-points crossover it is possible to transform the population $P_1$ into $P_2$. We also have that $P_1 \Pxo{2} P_2'$, since $P_2'$ contains fewer elements than $P_2$. Since the addition of new genetic material in $P_1'$ with respect to $P_1$ does not impede the generation of the individuals that were already possible to generate by $P_1$, we have that $P_1' \Pxo{2} P_2'$, as desired.
\end{example}

The main idea is to define a \Cech closure according to Equation~\eqref{eq:iterated-closure} over $\Pop$ such that $\forall P \in \Pop$ and $\forall i \in \N$, $\cl{P}^i$
is the set of populations that can be obtained from $P$ after $i$ generations using only the crossover as a genetic
operator. To satisfy those requirements we defined a closure $\clxOp$ such that $\forall P_1,P_2 \in \Pop$:
\begin{enumerate}
  \item $P_2 \in \clx{P_1}$ iff $P_2$ can be obtained by using only crossover operations from $P_1$.
  \item If $P_2 \in \clx{P_1}$, the minimal $i \in \N$ such that $\clx{P_1} = \cl{P}^i$ is also the minimal number of generations needed to obtain $P_2$ from $P_1$.
\end{enumerate}

Such a closure, defined in~\cite{MV12} for one-point
crossover, is here generalized as follows.

\begin{definition}
  The \emph{crossover closure} for $n$-point crossover is a function $\clOp : \Parts{\Pop} \to \Parts{\Pop}$ defined, for every $A \subseteq \Pop$ as:
  \begin{enumerate}
    \item When $A = \emptyset$, $\cl{A} = \emptyset$.
    \item When $A = \Set{P}$, $\cl{\Set{P}} = \Set{P'
      \in \Pop \;|\; P \Pxo{n} P'}$.
    \item Otherwise, $\cl{A} = \bigcup_{P \in A}
    \cl{\Set{P}}$.
  \end{enumerate}
\end{definition}

The following two propositions, as a generalization of the corresponding results relative to the case for $n=1$~\cite{MV12}, holds for a \emph{crossover closure}.

\begin{proposition}
  The crossover closure is a \Cech closure.
\end{proposition}

\begin{proposition}
  \label{prop:sequence}
  For all $P_1,P_2 \in \Pop$ and for all $k \in \N$, $P_2 \in \cl{\Set{P_1}}^k$ iff there exists $Q_0=P_1,Q_1,\ldots,Q_{k-1},Q_k=P_2 \in \Pop$ such that $\forall i \in [0,k)\;Q_i\Pxo{n} Q_{i+1}$.
\end{proposition}

Intuitively, the previous proposition states that verifying that a population $P_2$ is inside the $k^\text{th}$
iteration of the closure of a population $P_1$ is the same as verifying if it is possible to obtain $P_2$ starting from
$P_1$ in $k$ generations using only crossover operations.

We are now going to show that from the closure of a population $P_1$ it is always possible to find a particular
population $P'$ such that $\cl{\Set{P_1}}^2 = \cl{\Set{P'}}$. That is, we can always focus on considering closures of
singletons. We define, $\forall P \in \Pop$ and $\forall i \in \N$, the set $S_i(P) \in \Pop$ as follows:
\[
S_i(P) =
\begin{cases}
  P & \text{if } i = 0\\
  \bigcup \cl{\Set{S_{i-1}(P)}} & \text{otherwise}
\end{cases}
\]

The following proposition links the iteration of the closure with the sequence $S_0(P),S_1(P),\ldots$.

\begin{proposition}
  For all $P_1,P_2 \in \Pop$ such that $\exists i \in \N$ with $P_2 \in \cl{\Set{P_1}}^i$ the following holds:
  \[
  \min\lbrace i \in\ \N \;|\; P_2 \in \cl{\Set{P_1}}^i \rbrace =
  \begin{cases}
    1 & \text{if } P_2 \subset P_1\\
    \min\lbrace i \in \N \;|\; P_2 \subseteq S_i(P_1)\rbrace & \text{otherwise}
  \end{cases}
  \]
\end{proposition}

The previous proposition intuitively states that it is possible to know the minimum number of generations needed to
obtain a population from another by only considering a \Cech closure of a particular singleton set.

\subsection{Distance Definition}

From the previous definitions, we have the elements to define a metric between populations. The definitions
of~\cite{MV12} can be easily adapted to the $n$-points crossover case.

Let $k^* = \min\lbrace k \in \N \;|\; \forall U \subseteq \Pop \;:\; \cl{U}^k = \cl{U}^{k+1}\rbrace$ (\ie the minimum
number of iteration of the \Cech closure needed to reach a fixed point independently from the starting set). Then a
quasi-metric (\ie a distance without the symmetry property, simply called direction distance in this paper) $f_P:\Pop \times \Pop \to \R_+$ can be defined as:
\[
f_P(P_1,P_2) =
\begin{cases}
  \min\lbrace k \in \N \;|\; P_2 \in \cl{\Set{P_1}}^k \rbrace & \text{if }P_2 \in \cl{\Set{P_1}}^{k^*}\\
  k^* & \text{otherwise}
\end{cases}
\]
To obtain a distance between populations, the function $d_P$ defined as $(P_1,P_2) \mapsto \frac{1}{2} \left( f_P(P_1,P_2) + f_P(P_2,P_1) \right)$ suffices. For any fixed $P \in \Pop$ it is possible to define a distance $\delta_P$ between elements of $\Sigma^\ell$ as:\\
\[
\delta_P(x,y) = d_P\left(\left( P \setminus \Set{x} \right) \cup \Set{y}, \left( P \setminus \Set{y} \right) \cup \Set{x}\right).
\]
In the experimental part of the paper, however, we will use the direction distance $f_p$ since it is more consistent with our idea of the dynamics of crossover.

\section{A Representation for Populations as Lower Sets}
\label{sec:representation_for_populations}

In this section, a way to represent populations as lower sets is introduced. This representation allows us to compute the
previously defined distance in an efficient way (\ie in a time that is polynomial \wrt the size of the populations and
the length of the individuals).

While $\SC{1}$ is a lattice, $\forall n > 1$, $\forall \ell > 2 \lfloor \frac{n}{2} \rfloor + 3$, the $\SC{n}$ poset is
not a lattice (\wrt set inclusion).

\begin{example}
  For example, consider, for a fixed $n$ and $\ell = 2 \lfloor \frac{n}{2} \rfloor + 3$, let $A,B \in \SC{n}$ where $A = \Set{1} \cup \Set{5} \cup \Set{7} \cup \Set{9} \ldots \cup \Set{2 \lfloor \frac{n}{2} \rfloor + 3}$ and $B = \Set{1} \cup \Set{3} \cup \Set{7} \cup \Set{9} \ldots \cup \Set{2 \lfloor \frac{n}{2} \rfloor + 3}$. Recall that $\SC{n}$ cannot have elements that are the union of more than $n$ disjoint sets in the form $[i,j]$ and both $A$ and $B$ are union of $n$ disjoint sets. It is immediate that the \emph{atomic} upper bound of $A$ and $B$ is not unique, since both $ \Set{1} \cup [3,5] \cup \Set{7} \cup \Set{9} \ldots \cup \Set{2 \lfloor \frac{n}{2} \rfloor + 3}$ and $ [1,3] \cup \Set{5} \cup \Set{7} \cup \Set{9} \ldots \cup \Set{2 \lfloor \frac{n}{2} \rfloor + 3}$ are \emph{atomic} upper bounds. Hence they are not the least upper bound, that, by definition, must be unique.
\end{example}

From now on, we fix $n,\ell \in \N$. We now recall some definitions from~\cite{MV12} adapting them to the $n$-points
crossover case.

\begin{definition}
  \label{def:represented}
  Let $x \in \Sigma^\ell$, $A \in \SC{n}$ and $P \in \Pop$. We say that $A$ is represented in $P$ iff $\exists y \in P$ such that $\forall a \in A$ $y_a = x_a$.
\end{definition}

The concept of representation has been extended to populations:

\begin{definition}
  Fix $x \in \Sigma^\ell$. We define $\rxOp: \Pop \to \Parts{\SC{n}}$ as $\rx{P} = \Set{A \in \SC{n} \;|\; A \text{ is represented in } P}$.
\end{definition}

\begin{proposition}
  For all $P \in \Pop$ and for all $x \in \Sigma^\ell$, $\rx{P}$ is a lower set of $\SC{n}$.
\end{proposition}
\begin{proof}
  Let $A \in \rx{P}$. Hence there exists $y \in P$ such that $A$ is represented in $\{y\}$. It follows from Definition~\ref{def:represented} that any $B \subseteq A$ is also represented in $\{y\}$ and, as a consequence, it is represented in $P$. Hence, $\rx{P}$ is a lower set in $\SC{n}$.
\end{proof}

We now define the notion of the \emph{alternating number} of two elements of $\SC{n}$. The idea is that given $A,B \in
\SC{n}$ we want to find an algorithms that given any two strings $y,z \in \Sigma^\ell$ such that $A \in \rx{\{y\}}$ and
$B \in \rx{\{z\}}$, generates a string $w \in \Sigma^\ell$ such that $A,B \in \rx{\{w\}}$ by scanning the string left to
right and choosing at every position $i \in [1,\ell]$ either $y_i$ or $z_i$. The \emph{alternating number} is the
minimum number of ``switch'' from copying one string to copy the other that such an algorithm must perform when $A$ and $B$
are fixed. Intuitively, if such a number is less than the number of available crossover points then the string $w$ can
be generated by one crossover operation starting from two strings, the first one having $A$ in its representation and the
second one having $B$.

\begin{definition}
  Let $x \in \Sigma^\ell$, and let $A,B\in \SC{n}$. The \emph{crossover language of  these $A$ and $B$}, denoted by $\mathcal{L}_{A,B} \subseteq \Gamma^\ell$ for the alphabet $\Gamma = \{a,b\}$, is defined as:
  \begin{align*}
    \forall w \in \{a,b\}^\ell &&  w \in \mathcal{L}_{A,B} \iff \forall i \in [1,\ell]\;
    & (i \in A \wedge i \notin B \implies w_i = a) \\
    && \wedge & (i \notin A \wedge i \in B \implies w_i = b)
  \end{align*}
  The \emph{alternating number} of $\mathcal{L}_{A,B}$ (denoted by $\alt{\mathcal{L}_{A,B}}$) is the smallest $m \in \N$ such that there exists $w \in \mathcal{L}_{A,B}$ with where the symbols $a$ and $b$ alternates $m$ times.
\end{definition}

We are now going to define a function remapping lower sets of $\SC{n}$ to lower sets of $\SC{n}$. Intuitively, this
function will transform the representation of a population $P$ in the representation of another population that is the
union of all the populations in the closure of $P$.

\begin{definition}
  We define $\mclOp: \DownSets{\SC{n}} \to \DownSets{\SC{n}}$ as follows. For all $U \in \DownSets{\SC{n}}$
  \[
  \mcl{U} = \lbrace A \in \SC{n} \;|\; \exists B_1,B_2 \in U \text{ s.t. } A = B_1 \cup B_2 \text{ and } \alt{\mathcal{L}_{B_1,B_2}} \le n \rbrace
  \]
\end{definition}

The main result is the following since it allows us to compute the minimum number of iterations of the closure
necessary to obtain a certain element of $\Sigma^\ell$ by iterating the function $\mclOp$.

\begin{proposition}
  For all $x \in \Sigma^\ell$ the following diagram commutes:
  \[
  \xymatrix{
    \Pop \ar[rr]^{P \mapsto \bigcup \cl{\Set{P}}} \ar[d]_{\rxOp} && \Pop \ar[d]^{\rxOp}\\
    \DownSets{\SC{n}} \ar[rr]^{\mclOp} && \DownSets{\SC{n}}
  }
  \]
\end{proposition}
\begin{proof} Fix $x \in \Sigma^\ell$, $P \in \Pop$, and $A \in\SC{n}$. The proof is divided into two parts:\\
  1) $A \in \rx{\bigcup\cl{\{P\}}} \Rightarrow A \in \mcl{\rx{P}}$.\\
  Let $A$ be in $\rx{\bigcup\cl{\{P\}}}$. Then there exists $y \in \bigcup\cl{\{P\}}$ obtained by the $n$-points crossover of two elements $z,v \in P$ and such that $A \in \rx{\{y\}}$. Consider the word $w \in \{a,b\}^\ell$ defined as:
  \[
  \forall i \in [1,\ell] \qquad w_i =
  \begin{cases}
    a & \text{if $y_i$ was taken from $z$}\\
    b & \text{if $y_i$ was taken from $y$}
  \end{cases}
  \]
  Since $y$ has been obtained by $n$-points crossover, the word has at most $n$ alternations. Furthermore, there exists $B \in \rx{\{z\}}$ and $C \in \rx{\{v\}}$ such that $A = B \cup C$ and such that $w \in \mathcal{L}_{B,C}$ (it follows the fact that $A$ is in the representation of $\{y\}$ and that it has been obtained from the crossover of $z$ and $v$). Hence, $A \in \mcl{\rx{P}}$.

  2) $A \in \mcl{\rx{P}} \Rightarrow A \in \rx{\bigcup\cl{\{P\}}}$.\\
  Let $A \in \mcl{\rx{P}}$. Then there exists $B_1,B_2 \in \rx{P}$ such that $B_1 \cup B_2 = A$ and $\alt{\mathcal{L}_{B_1,B_2}} \le n$. By the definition of $\rxOp$ there exists $z,v \in P$ such that $B_1 \in \rx{\{z\}}$ and $B_2 \in \rx{\{v\}}$. We claim that there exists $y \in \rx{\bigcup\cl{\{P\}}}$ obtained from the $n$-points crossover of $z,v$ such that $A \in \rx{\{y\}}$. Let $w = a^{k_1}b^{k_2}\ldots b^{k_h} \in \mathcal{L}_{B_1,B_2}$ be a word with $\mathcal{L}_{B_1,B_2}$ with $h \le n$ (this word exists by hypothesis). It is possible to see that by choosing as crossover points between $z$ and $v$ the positions $k_1,\ldots,k_h$, the obtained element $y$ is such that $A \in \rx{\{y\}}$.
\end{proof}

With the observation that $\forall P \in \Pop$ and $\forall x \in \Sigma^\ell$, $[1,\ell] \in \rx{P}$ iff $x \in P$ we
can conclude that for all $P \in \Pop$ and for all $x \in \Sigma^\ell$, $\min\lbrace k \in \N \;|\; [1,\ell] \in
\mclIter{\rx{P_2}}{k}\rbrace$ is equal to $\min\lbrace k \in \N \;|\; \Set{x} \subseteq S_k(P) \rbrace$. Notice that the
previous proposition also implies that $\mclOp$ also remaps lower sets to lower sets, a condition that was not proved
when the function was defined.

\begin{remark}
  \label{rmk:fixed_point}
  Note that since the function $\mclOp$ is such that $\forall A \in \Parts{\SC{n}}$, $\mcl{A} \supseteq A$ and the poset $\SC{n}$ is finite, the dynamics induced by the iteration of $\mclOp$ always reaches a fixed point (\ie an equilibrium point). Trivially, there are no cyclic points.
\end{remark}

\subsection{The computational complexity of computing the distance between two populations}

The computational complexity of determining the distance between two populations, using the proposed representation, is
polynomial in the size of the individuals and in the size of the populations, as we are going to show. The presented bound are not tight, but this is not necessary for showing that the computation can be performed in polynomial time.

Let $P_1$ and $P_2$ be two populations. To compute $f_P(P_1,P_2)$ we obtain the following time complexity bounds:
\begin{enumerate}
  \item For each element $x$ in $P_2$, it is necessary to build the poset $\SC{n}$, which has size $O(\ell^{2n})$ (\ie polynomial in the length of the individuals but exponential in the number of crossover points - that we have assumed to be fixed). Hence, the time required for this step is linear with respect to $|P_2|$ and polynomial with respect to $\ell$.
  \item For each partial order $\SC{n}$ with the associated element $x \in P_2$, it is necessary to computer $\rx{P_1}$, which can be performed by checking every individual in $P_2$ with every element of $\SC{n}$. Hence, the number of steps necessary will be, for each $x \in P_2$, polynomial with respect to the size of $\SC{n}$ (and, hence, with respect to $\ell$), and $|P_1|$.
  \item Finally, computing $\mclOp$ is polynomial with respect to the size of $\SC{n}$, since it can be performed by checking all the pairs of elements in $\SC{n}$. Since $\SC{n}$ is monotone, it cannot be iterated more than $|\SC{n}|$ times before reaching a fixed point, thus still giving a polynomial time bound. In fact, by adapting a result in~\cite{MV12}, it is possible to show that the number of iterations is at most logarithmic with respect to $\ell$.
\end{enumerate}
In conclusion, $f_p(P_1,P_2)$ can be computed in polynomial time with respect to the size $\ell$ of the individuals and the size of the $P_1$ and $P_2$. More precise bounds can be obtained by exactly specifying the data structures and representations used while implementing the algorithm.

\section{Experimental Results on the Distance Distribution}
\label{sec:experiments}

In this section, we perform a comparison of the difference distances distribution obtained for a different number of crossover points on an $8$-bit individual. This experimental exploration is necessary to check if the proposed distance is significantly different for an increasing number of crossover points. Intuitively, a higher number of crossover points should increase the ability to produce new individuals, thus decreasing the average distance.

One first obstacle in the experimental design is to determine how to compute a distance between two individuals when no population is given. Therefore, in order to calculate the (directional) crossover distance, for each one of the $2^{8}$ possible individuals we have used $100$ small populations of~$4$ individuals each (generated randomly), to which we have added the individual for which we want to estimate the distance to the optimum, i.e., a population consisting only of the individual $11111111$. The reported distance measure is then, for each individual, the average of the $100$~different values obtained in this way.

The results are presented from~Figure~\ref{fig:distance-2} to~Figure~\ref{fig:distance-7}. Each figure shows a comparison of the distance distribution for $1$-point and a multiple points crossover (from $2$ crossover points in Figure~\ref{fig:distance-2} to a maximum of $7$ crossover points in Figure~\ref{fig:distance-7}). The average and variance of each distance distribution are summarized in Table~\ref{tab:stats} and Figure~\ref{fig:stats}.

\begin{table}[h]
  \centering
  \begin{tabular}{rrrrrrrr}
    & $n = 1$ & $n = 2$ & $n = 3$ & $n = 4$ & $n = 5$ & $n = 6$ & $n = 7$ \\
    \hline
    \textbf{Average} & $2.3771$ &  $2.2284$ &  $2.0211$ & $1.9536$ & $1.9387$ & $1.9364$ & $1.9363$ \\
    \textbf{Variance} & $0.1026$ & $0.0923$ & $0.1021$ & $0.0999$ & $0.0983$ & $0.0980$ & $0.0979$ \\
  \end{tabular}
  \caption{\label{tab:stats} The average and the variance for a different number of crossover points on an $8$-bit string. The average distance decreases and appears to converge at a value of approximately~$1.93$ from the top (upper approximation).}
\end{table}

\begin{figure}[h]
  \centering
  \includegraphics[width=0.48\textwidth]{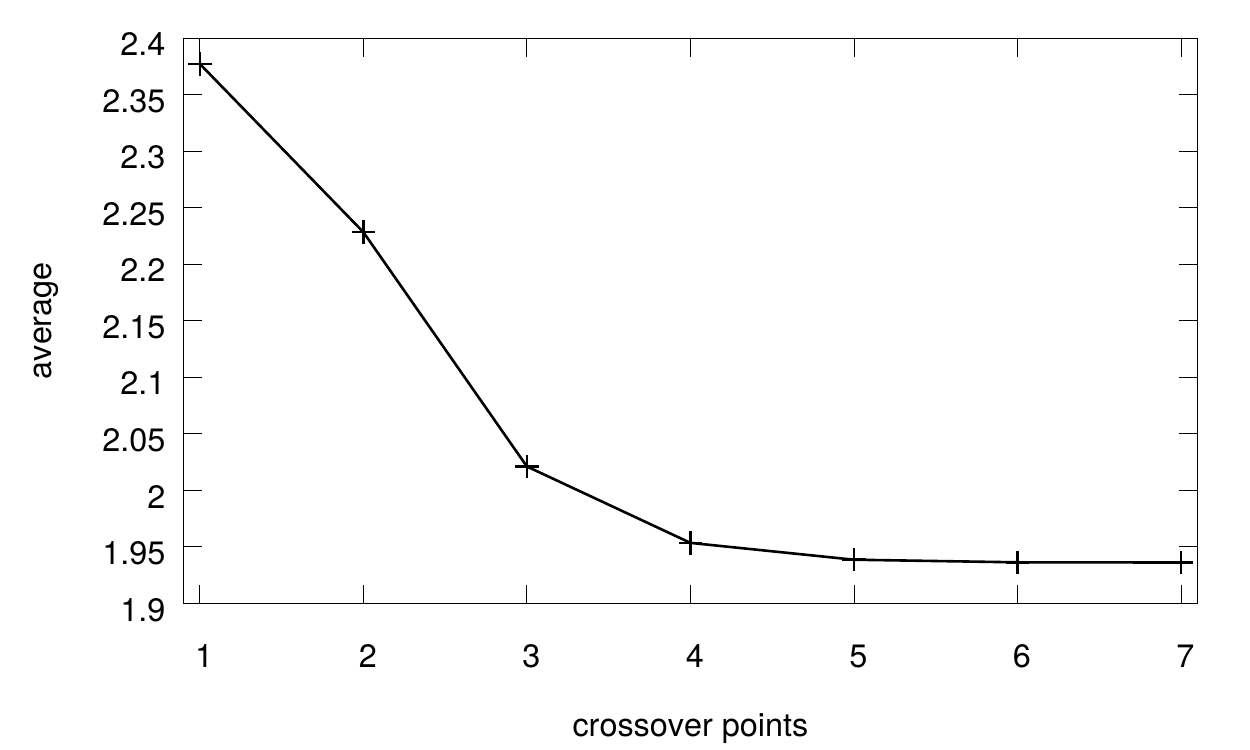}\hfill
  \includegraphics[width=0.48\textwidth]{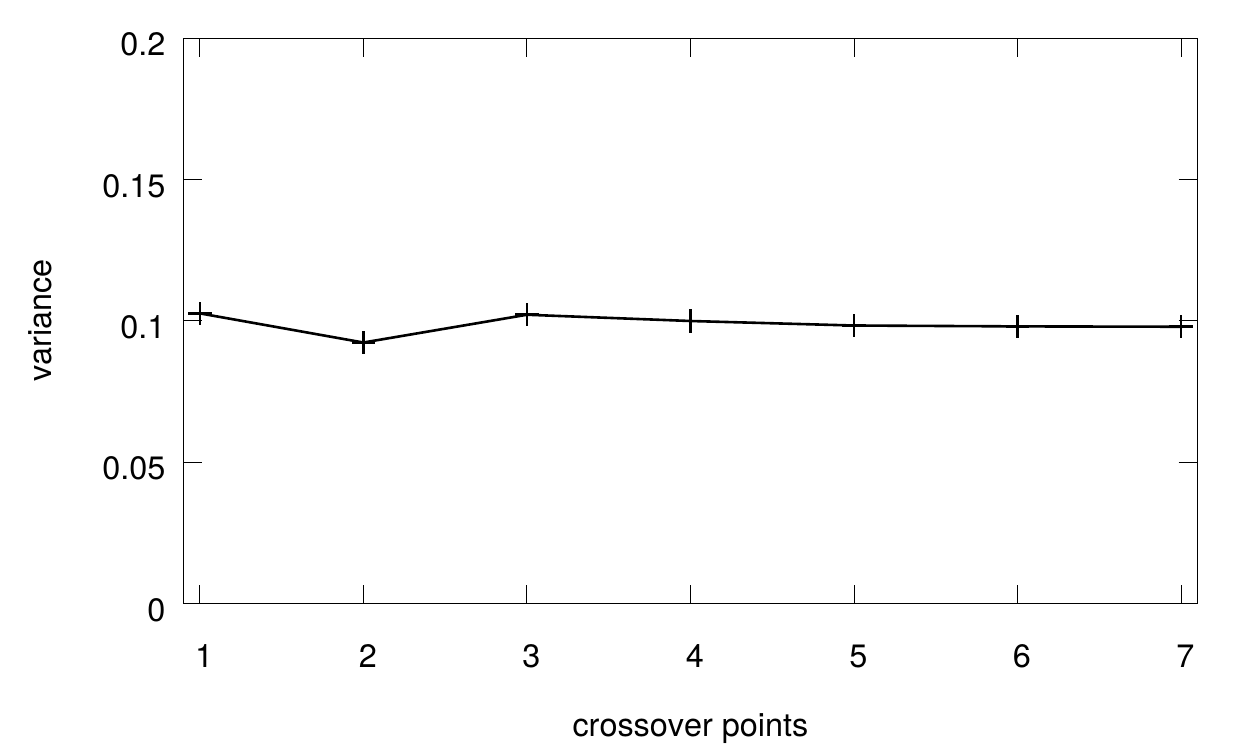}
  \caption{\label{fig:stats} How the average and the variance change when the number of crossover points increases. While the former decreases, converging to $1.93$, the latter remains stable.} 
\end{figure}

As it is possible to observe, the shape of the distribution is similar to a Gaussian distribution in all the cases, with the obvious difference that the optimum has always distance $0$ from itself. This can be observed in more details in Figure~\ref{fig:gaussians} and Table~\ref{tab:peak}, where the fitting of the obtained results to a Gaussian distribution has been reported.

\begin{figure}[h]
  \centering
  \includegraphics[width=0.98\textwidth]{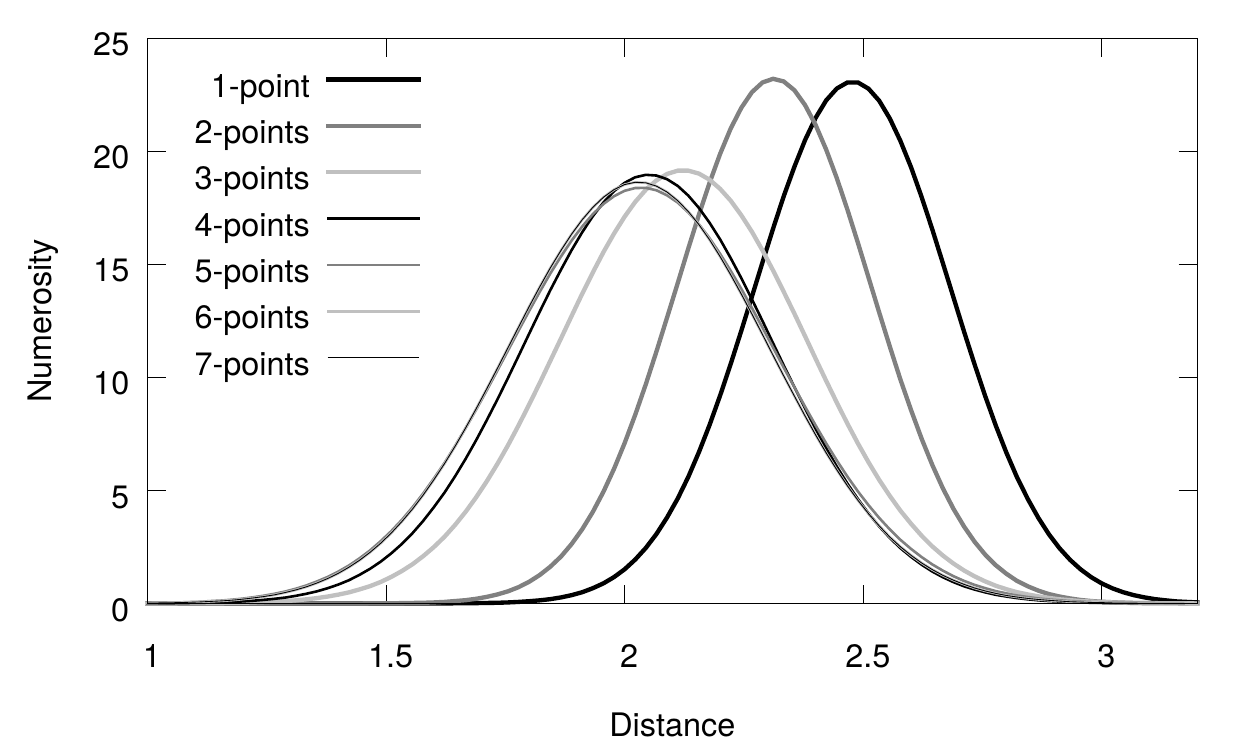}
  \caption{\label{fig:gaussians}A comparison of the distance distribution (each fitted to a Gaussian using bins of size $0.05$) from $1$ to $7$ crossover points.}
\end{figure}

\begin{table}[h]
  \centering
  \begin{tabular}{rrrrrrrr}
    & $n = 1$ & $n = 2$ & $n = 3$ & $n = 4$ & $n = 5$ & $n = 6$ & $n = 7$ \\
    \hline
    \textbf{Peak value} & $23.068$ & $23.196$ &  $19.154$ & $18.960$ & $18.386$ & $18.625$ & $18.625$ \\
  \end{tabular}
  \caption{\label{tab:peak}The peak values of the fitted Gaussians (using bins of size $0.05$) for a different number of crossover points on an $8$-bit string.}
\end{table}

The average decreases monotonically with the increase of crossover points used, up to $7$, the maximum possible for $8$-bits individuals. In particular, the average appears to converge to a value of $1.93$, where no possible improvements are possible. While the improvements from $1$ to $2$ and from $2$ to $3$ crossover points are quite large, successive increases in the number of crossover points do not produce improvements of a similar magnitude. Therefore, we can observe that there are diminishing returns when increasing the number of crossover points. This is intuitively explainable in the following way: moving from $1$ to $2$ crossover points greatly augments the possibilities to generate new individuals in less time for many cases (e.g., $11100111$ and $00011000$ can be used to generate the optimum in one step with $2$-points crossover, while at least two steps are required for $1$-point crossover). When the number of crossover points is already high, the cases in which additional crossover points can actually decrease the number of generations necessary to reach the optimum are fewer and their contribution to the average is therefore reduced.

\begin{figure}
  \centering
  \includegraphics[width=0.98\textwidth]{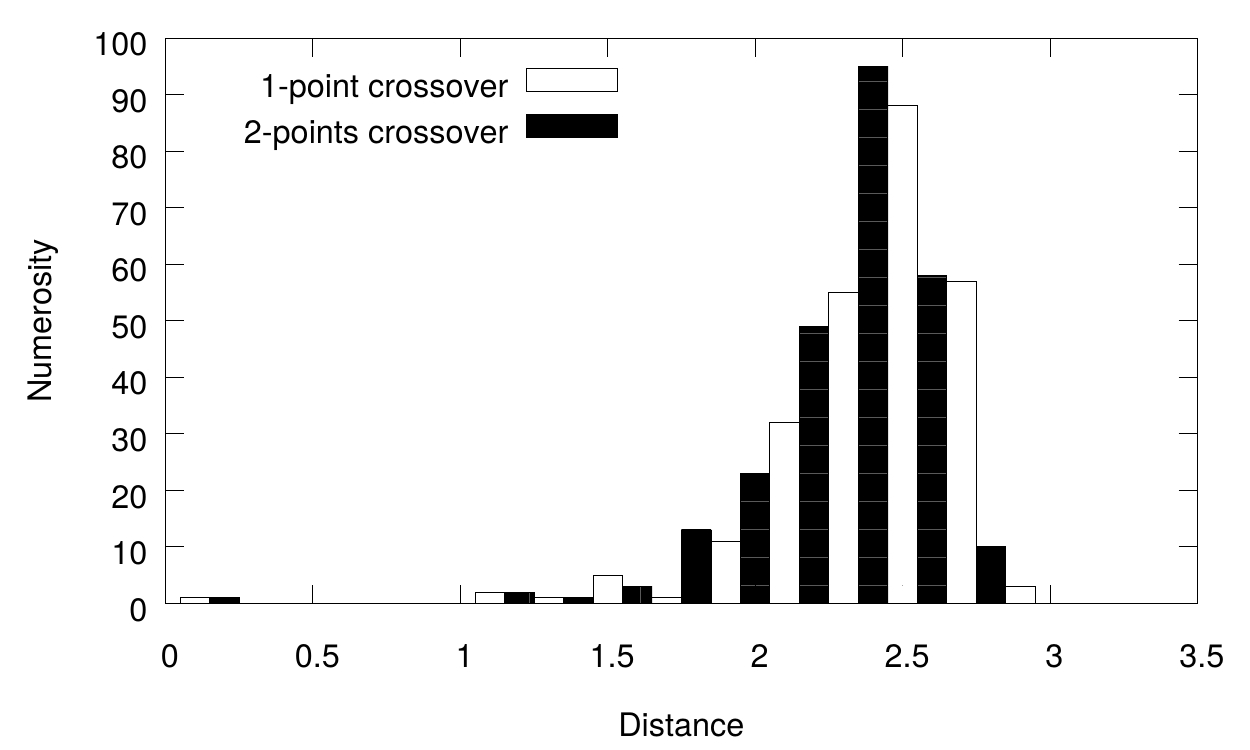}
  \caption{\label{fig:distance-2}A comparison of the distance distribution between the $2$-points and $1$-point crossovers.}
\end{figure}

\begin{figure}
  \centering
  \includegraphics[width=0.98\textwidth]{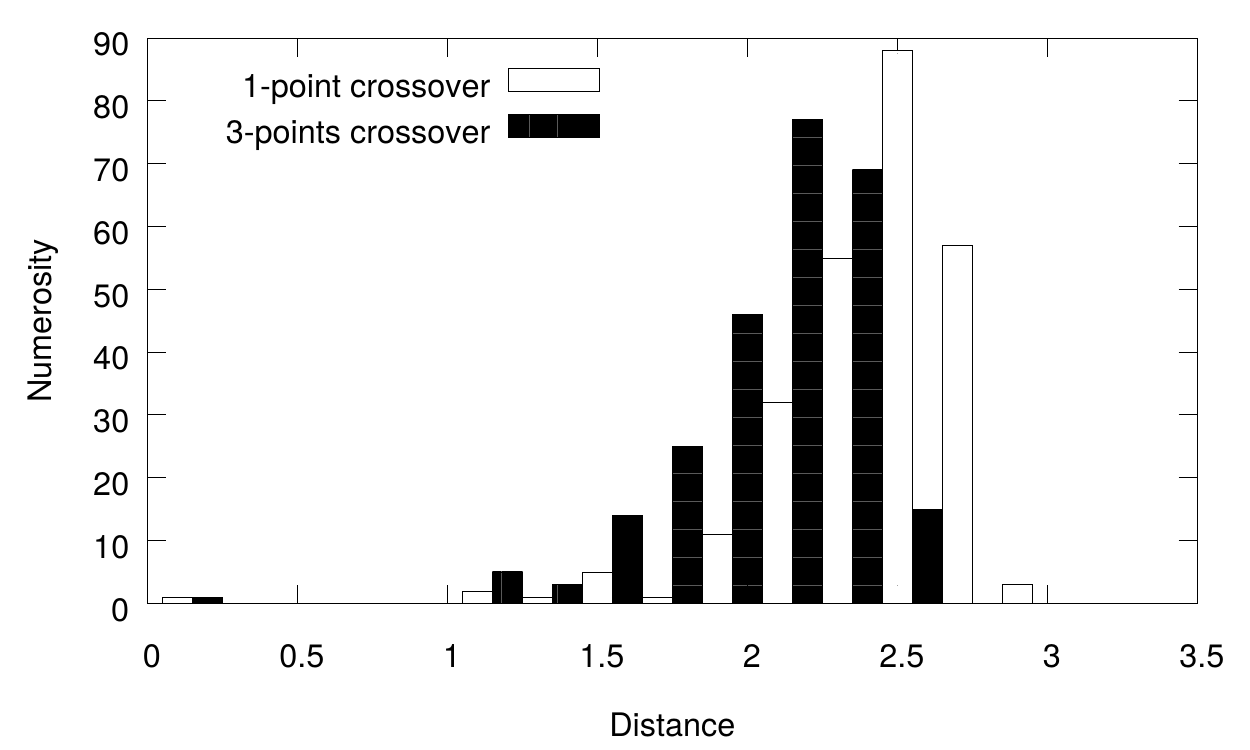}
  \caption{\label{fig:distance-3}A comparison of the distance distribution between the $3$-points and $1$-point crossovers.}
\end{figure}

\begin{figure}
  \centering
  \includegraphics[width=0.98\textwidth]{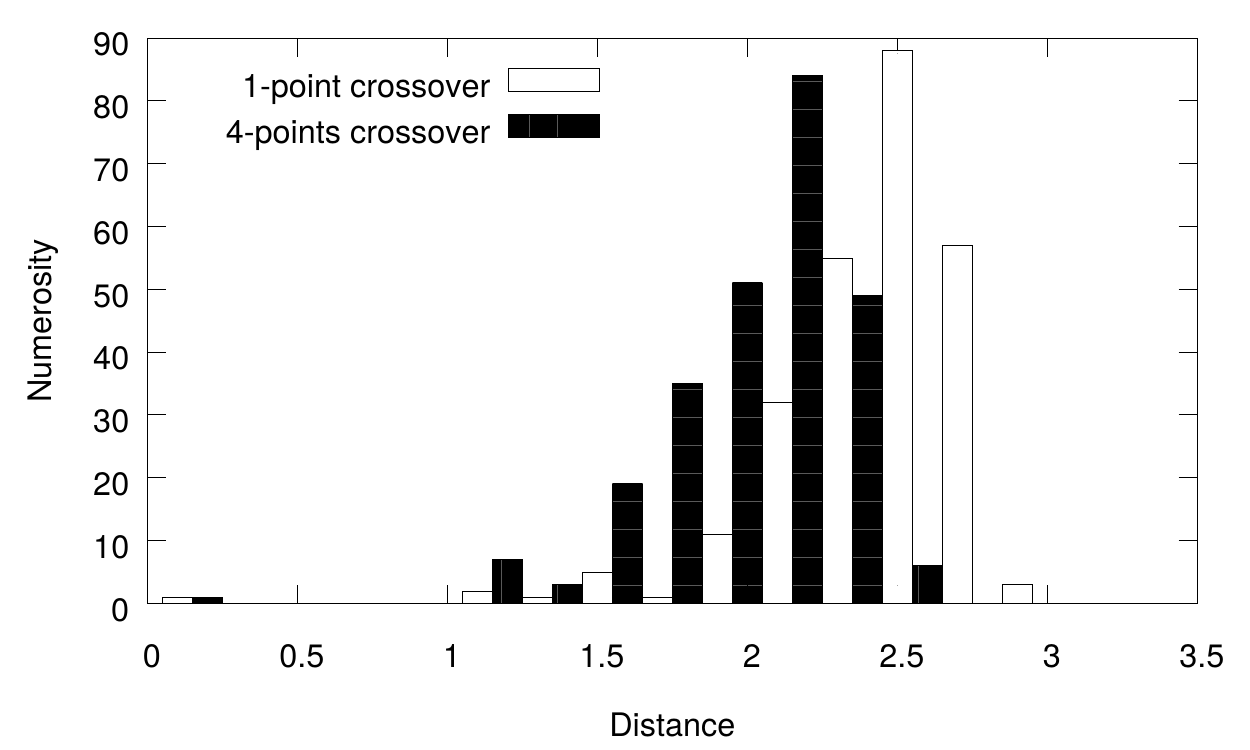}
  \caption{\label{fig:distance-4}A comparison of the distance distribution between the $4$-points and $1$-point crossovers.}
\end{figure}

\begin{figure}
  \centering
  \includegraphics[width=0.98\textwidth]{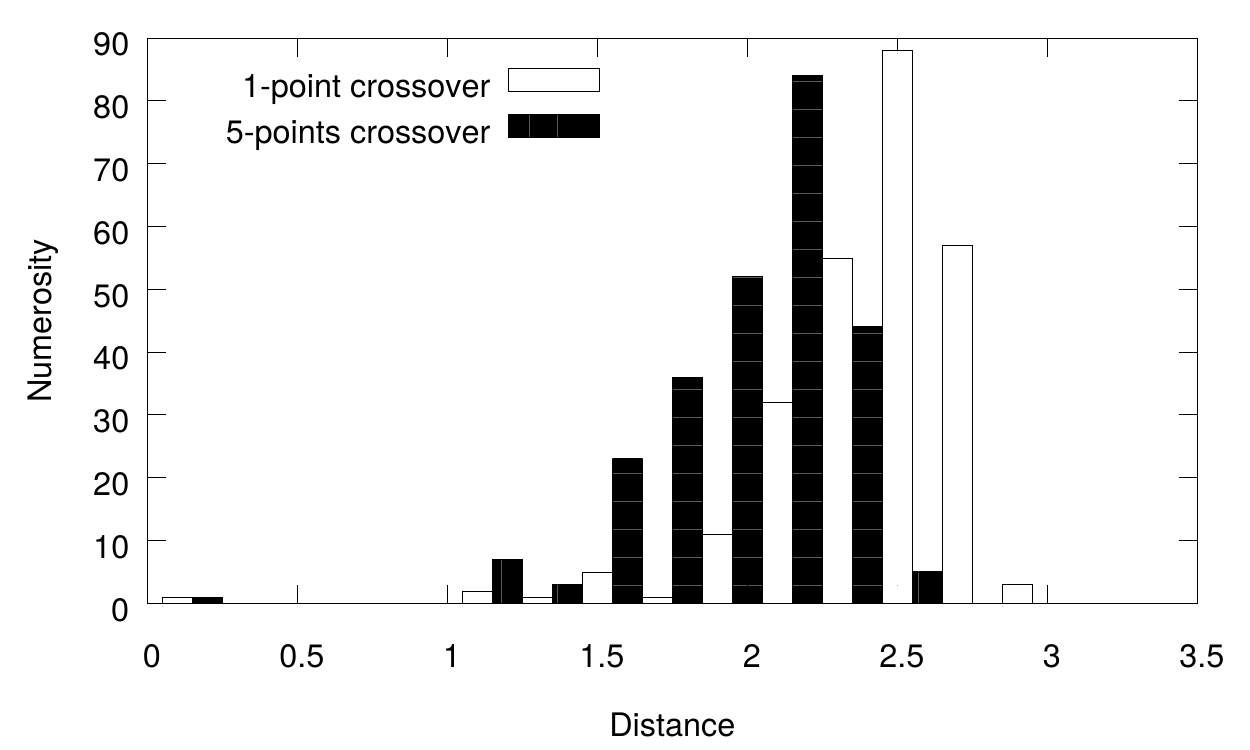}
  \caption{\label{fig:distance-5}A comparison of the distance distribution between the $5$-points and $1$-point crossovers.}
\end{figure}

\begin{figure}
  \centering
  \includegraphics[width=0.98\textwidth]{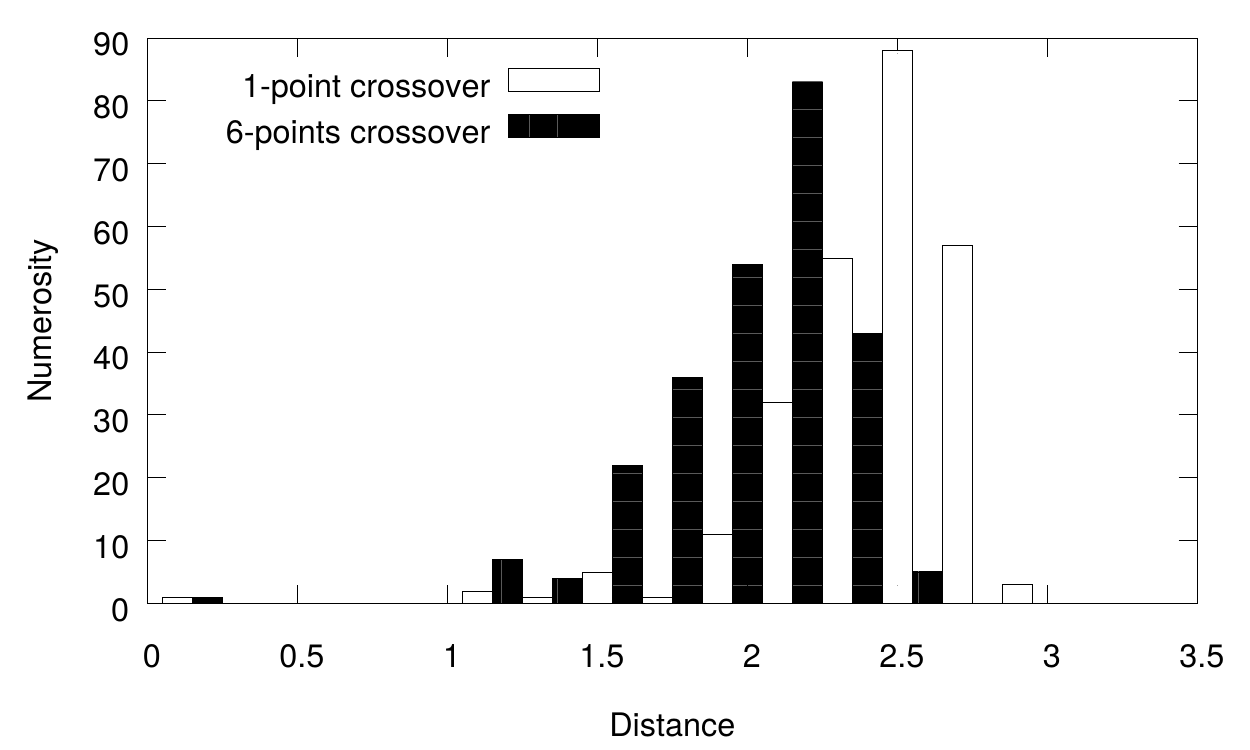}
  \caption{\label{fig:distance-6}A comparison of the distance distribution between the $6$-points and $1$-point crossovers.}
\end{figure}

\begin{figure}
  \centering
  \includegraphics[width=0.98\textwidth]{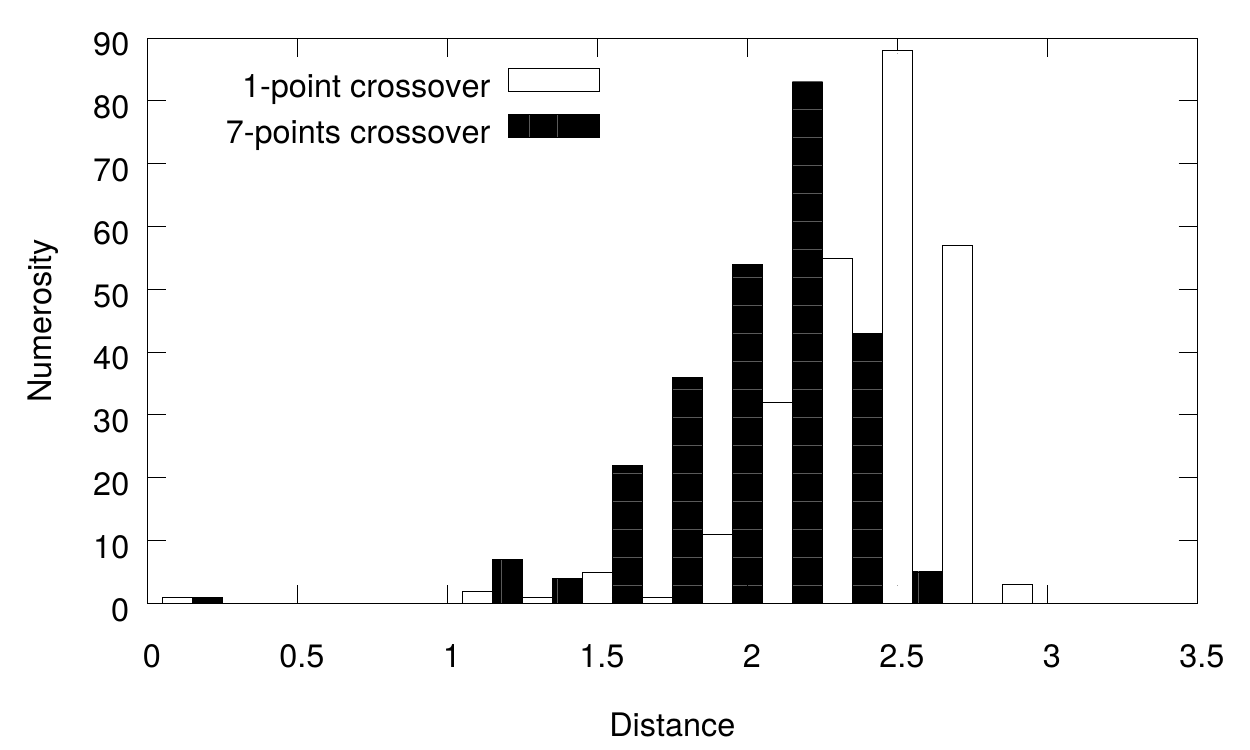}
  \caption{\label{fig:distance-7}A comparison of the distance distribution between the $7$-points and $1$-point crossovers.}
\end{figure}

It is interesting to remark that $n$-points crossover can be considered as a ``parallel version'' of one-point crossover, in which $n$ one-point crossover operations take place in parallel (as can be seen in Example~\ref{ex:xo_rel}). Hence, the study of the relations between different distances can be interesting to better understand the effects of this parallelization.

\section{Further Remarks and Contributions}
\label{sec:further_remarks}

In this paper, a recent model for one-point crossover in GA has been generalized to $n$-points crossover. We have shown
that when the kind of crossover is fixed, the distance can be computed in polynomial time \wrt both population size and
individual length. This result indicates that the structures used for modeling one-point crossover can be generalized to
deal with $n$-points crossover. Hence, the proposed model is not limited to a specific case and the results on the
polynomial complexity in time can be extended to more general kinds of crossover. In order to experimentally study the proposed distance, we have showed how the distance distribution changes with different numbers of crossover points.  

Future works will involve a more in-depth study of this model and, in general, an investigation of the conditions that a certain structure must satisfy for modeling crossover in GAs. It will also be the focus of future studies to determine what is a good trade-off between minimizing the number of crossover points and minimizing the average of the distance value; it would be interesting to observe if there is a correlation between these variations on the average of the distance and the performance of a GA on synthetic or real-world problems. Finally, a general way of extending this model to other evolutionary algorithms should be devised.

\bibliographystyle{elsarticle-num}
\bibliography{bibliography}

\end{document}